\documentclass[10pt]{article}

\usepackage[margin=1in]{geometry}
\usepackage{comment}
\usepackage{algorithm}
\usepackage{algpseudocode}
\newcommand{\E}{\mathbb{E}}
\newcommand{\Var}{\textnormal{Var}}

\newcommand{\OPTc}{\mathrm{OPT}_{\mathrm s}}
\newcommand{\OPTd}{\mathrm{OPT}_{\mathrm d}}

\usepackage{algorithm, algorithmicx, algpseudocode}
\usepackage{amssymb}
\usepackage{amsmath}
\usepackage{graphicx}
\usepackage{mathtools}
\usepackage{needspace} 
\usepackage{multicol}
\usepackage{listings}
\usepackage{amsthm}
\usepackage{thmtools} 
\usepackage[shortlabels]{enumitem}
\usepackage[italicdiff]{physics}
\usepackage{cancel}
\usepackage[dvipsnames]{xcolor}
\usepackage{verbatim}
\usepackage{comment}
\usepackage{array}
\usepackage{tikz-cd}
\usepackage{import}
\usepackage{booktabs}
\usepackage{comment}
\usepackage{array}
\usepackage[colorlinks, linkcolor=blue, citecolor=blue, linktoc=page]{hyperref}
\usepackage{booktabs}
\PassOptionsToPackage{table}{xcolor}
\usepackage{array}
\usepackage{adjustbox}
\usepackage{tikz}
\usetikzlibrary{patterns.meta}

\newcommand{\ourresult}[1]{%
	\tikz[baseline=(X.base)]{
		\node[
		rounded corners=1pt,
		fill=blue!6,
		pattern={Lines[angle=45,distance=2.4pt,line width=0.12pt]},
		pattern color=blue!45,
		inner xsep=3pt,
		inner ysep=1.4pt,
		minimum height=3.2ex,
		text width=4.15cm,
		align=center
		] (X) {$\displaystyle #1$};
	}%
}
\usepackage{forest}
\usepackage{float}
\usepackage{xspace}
\usepackage{aliascnt}
\usepackage{wrapfig}
\usepackage{thm-restate}
\usepackage{adjustbox}
\usepackage{bm}
\usepackage{bbm}
\usepackage{makecell}
\usepackage{nicefrac}
\usepackage[utf8]{inputenc} 
\usepackage[T1]{fontenc}    
\usepackage{hyperref}       
\usepackage{url}            
\usepackage{booktabs}       
\usepackage{amsfonts}       
\usepackage{nicefrac}       
\usepackage{microtype}      
\usepackage{thm-restate}
\usepackage{cleveref}

\newtheorem{theorem}{Theorem}
\numberwithin{theorem}{section}
\newtheorem*{theorem*}{Theorem}

\newtheorem{corollary}[theorem]{Corollary}

\theoremstyle{definition}

\newfloat{protocol}{tbp}{lop}
\floatname{protocol}{Protocol}
\allowdisplaybreaks

\usepackage{tikz}

\newcommand{\circled}[1]{%
	\tikz[baseline=(char.base)]{
		\node[shape=circle,draw,inner sep=2pt,line width=0.4pt] (char) {$#1$};
	}%
}

\usepackage{natbib}

\usepackage{caption}
\makeatletter
\renewcommand\thanks[1]{%
  \footnotemark%
  \protected@xdef\@thanks{\@thanks
    \protect\footnotetext[\the\c@footnote]{#1}}%
}
\makeatother

\title{Multi-Armed Bandits With Best-Action Queries}
\author{
	Francesco Bacchiocchi\thanks{Alphabetical order.}\\[-0.2em]
	\texttt{francesco.bacchiocchi@polimi.it} \\[-0.2em]
	Politecnico di Milano 
	\and
	Matteo Castiglioni\footnotemark[1] \\[-0.2em]
	\texttt{matteo.castiglioni@polimi.it} \\[-0.2em]
	 Politecnico di Milano 
	\and
	Alberto Marchesi\footnotemark[1] \\[-0.2em]
	\texttt{alberto.marchesi@polimi.it} \\[-0.2em]
	Politecnico di Milano
	\and
	Francesco Emanuele Stradi\footnotemark[1]\\[-0.2em]
	\texttt{francescoemanuele.stradi@polimi.it} \\[-0.2em]
	Politecnico di Milano 
}
\date{\today}

\begin{document}

\maketitle

\begin{abstract}
	We study \emph{multi-armed bandits} (MABs) augmented with \emph{best-action queries}, in which the learner may additionally query an oracle that reveals the best arm in the current round. This setting was recently characterized by~\citet{russo2024online} in the \emph{full-feedback} model, where the learner observes the rewards of all arms after each round. They show that, in both \emph{stochastic} and \emph{adversarial} environments, $k$ best-action queries reduce the optimal $\widetilde{\mathcal{O}}(\sqrt{T})$ regret to $\widetilde{\mathcal{O}}(\min\{T/k,\sqrt{T}\})$. Whether this improvement extends to the more realistic \emph{bandit-feedback} model---where the learner observes only the reward of the played arm---was left as an open problem. 
	We fully resolve this question. When rewards are stochastic but correlated among arms, we show that the full-feedback result does not carry over: 
	any algorithm must incur regret at least $\Omega(\sqrt{T-k})$. This lower bound directly extends to adversarial environments. On the positive side, we show that $\widetilde{\mathcal{O}}(\min\{T/k,\sqrt{T-k}\})$ regret is still achievable when rewards are stochastic and i.i.d., and establish a matching lower bound, up to logarithmic factors. Together, these results provide a complete characterization of the benefits of \emph{best-action queries} in the \emph{bandit-feedback} model.
\end{abstract}
\newpage

\tableofcontents

\newpage
\section{Introduction}

\emph{Multi-armed bandits} (MABs) constitute a canonical framework for sequential decision-making under uncertainty~\citep{auer2002finite,cesa2006prediction,noauthororeditor}. In a MAB problem, the learner repeatedly selects one among $n$ actions (arms) over a horizon of $T$ rounds, observing only the reward of the played arm. The goal of the learner is to maximize the cumulative reward collected during learning, or equivalently, to minimize the \emph{regret} with respect to the best fixed arm in hindsight. MABs capture a fundamental exploration-exploitation tradeoff and have thus become a standard model for applications such as recommendation systems~\citep{li2010contextual}, online experimentation~\citep{villar2015multi}, and digital advertising~\citep{chapelle2011empirical}.

Motivated by the recent growing literature on algorithms with machine learning-based predictions (see, \emph{e.g.}, the survey~\citep{mitzenmacher2022algorithms}), we study MABs augmented with \emph{best-action queries}~\citep{russo2024online}, in which the learner may query an oracle in at most \( k \) rounds, where the oracle reveals the best arm in the current round.
This models situations in which highly reliable information about the instantaneous optimal choice is available, but only through an expensive channel, such as expert intervention, high-fidelity simulation, offline evaluation, or a slower and more accurate decision pipeline. These high-quality interventions are typically too costly to invoke at every round because of monetary, computational, or latency constraints, and must therefore be used sparingly. From this perspective, best-action queries provide a clean abstraction for studying the value of \emph{rare but highly informative feedback} in online learning.
A concrete example is an online platform that continuously moderates user-generated content (\emph{e.g.}, Meta or Google). When a new post is created, the platform must decide whether to flag the content as harmful. The platform has two options: (i) relying on reviews by human experts (\emph{i.e.}, best-action queries), or (ii) invoking a learning algorithm that makes automated decisions. Clearly, due to budget constraints, human review is a scarce resource, and thus the platform can choose option (i) only a limited number of times.

\citet{russo2024online} fully characterize the impact of best-action queries in \emph{full-feedback} settings, where after each round the learner observes the rewards of all arms. They show that $k$ best-action queries improve the optimal regret rate to $\widetilde{\mathcal{O}}(\min\{T/k,\sqrt{(T-k)}\})$ in both stochastic and adversarial settings.\footnote{In this work, we use $\widetilde{\mathcal{O}}$ to hide logarithmic factors.}
It is quite surprising that a sublinear number of $T^\alpha$ queries, $\alpha\in (1/2,1)$, is sufficient to drop the regret to $\widetilde{\mathcal{O}}(T^{1-\alpha})$. For instance, $T^{2/3}$ queries are sufficient to drop the regret to $\widetilde{\mathcal{O}}(T^{1/3})$.
Whether similar results can be recovered under the more realistic \emph{bandit-feedback} settings---where only the reward of the played arm is observed---were left open by~\citet{russo2024online}. Understanding whether the same improvements can still be achieved under bandit feedback is substantially more challenging, since the learner can no longer rely on observing the outcomes of unplayed actions, and the informational value of a query must be assessed in a  partial-feedback regime. In this work, we address the following research question:
\begin{center}
	\emph{Can best-action queries improve the performance of an online learner\\ when only bandit feedback is available?}
\end{center}

We fully resolve this question. First, we establish a negative result: any algorithm with access to $k$ best-action queries must incur $\Omega(\sqrt{T-k})$ when reward are stochastic but \emph{correlated} accross arms. This lower bound naturally extends to \emph{adversarial} environments. Thus, in such settings, best-action queries cannot substantially improve the learner's performance.
To prove this, we construct two correlated instances that are not only difficult to distinguish but also limit the ``query power'' by simultaneously bounding the negative regret and the information gained about the optimal action. We also note that our result is tight: a simple $\widetilde{\mathcal{O}}(\sqrt{n(T-k)})$ upper bound is achievable by combining any standard adversarial regret minimizer with $k$ arbitrary oracle queries.

We then turn to the \emph{stochastic} setting with i.i.d.\ rewards. Here, the landscape changes significantly: we design an algorithm achieving regret $\widetilde{\mathcal{O}}(\min\{Tn/k,\sqrt{n(T-k)}\})$. This result stems from a key observation: the previous lower bound breaks in this regime because making two instances hard to distinguish requires the arm rewards to have large variance. This high variance, in turn, allows each query to yield a large negative per-round regret.

Our analysis formalize this observation by lowerbounding the expected negative regret of a query as a function of the variance of the underling distributions. This highlights the trade-off between two extremes: (i) the variance of the distributions is small, and hence a variance-aware learner can achieve regret significantly better than the standard $\mathcal{O}(\sqrt{T})$, (ii) the variance of the distributions is large, and hence, while the optimal regret in the rounds without query is the  standard $\mathcal{O}(\sqrt{T})$, the query compensate by providing a large negative regret.

Despite our analysis is fairly involved, the resulting algorithm is remarkably simple: we query the optimal arm in the first $k$ arms (ignoring the feedback), and then deploy a variance-aware regret minimizer, such as  UCB-V~\citep{audibert2009exploration},  for the remaining rounds.
Finally, we provide a matching lower bound of $\Omega(\min\{Tn/k,\sqrt{n(T-k)}\})$. Notice that the analogous lower bound for full-feedback~\citep{russo2024online} is inapplicable here as it relies on correlated rewards.

Overall, our results, summarized in Table~\ref{tab:main_results}, provide a complete characterization of the value of best-action queries under \emph{bandit feedback}, resolving the problem left open by~\citet{russo2024online}.

\begin{table}[t]
	\centering
	\footnotesize
	\setlength{\tabcolsep}{3.2pt}
	\renewcommand{\arraystretch}{1.15}
	\begin{adjustbox}{width=\columnwidth,center}
		\begin{tabular}{
				>{\raggedright\arraybackslash}p{3.0cm}
				@{\hspace{3pt}}
				c
				@{\hspace{0.8pt}}
				c
			}
			\toprule
			\textbf{Setting} & \textbf{Upper bound} & \textbf{Lower bound} \\
			\midrule
			Stoc. \emph{i.i.d.}
			& \ourresult{\widetilde{\mathcal{O}}\!\left(\min\!\left\{\frac{T}{k},\,\sqrt{T-k}\right\}\right)}
			& \ourresult{\Omega\!\left(\min\!\left\{\frac{T}{k},\,\sqrt{T-k}\right\}\right)} \\
			
			Stoc. \emph{correlated} ({Adv.})
			& $\widetilde{\mathcal{O}}\!\left(\sqrt{T-k}\right)$
			& \ourresult{\Omega\!\left(\sqrt{T-k}\right)} \\
			\bottomrule
		\end{tabular}
	\end{adjustbox}
	
	\vspace{4pt}
	\caption{Regret guarantees in MABs with best-action queries under bandit feedback.}
	\label{tab:main_results}
\end{table}

\subsection{Related Works}

In this section, we provide additional discussion on related works.

\paragraph{Hints and query-based online learning}
Best-action queries are related to the broader literature on online learning with hints. \citet{NIPS2017_22b1f2e0} study online linear optimization with hints that are correlated with the loss vectors, showing that such hints can lead to logarithmic regret over strongly convex domains. This line was extended to imperfect hints by \citet{pmlr-v119-bhaskara20a}, to sublinear hint budgets by \citet{bhaskara2021logarithmic}, and to bandit online linear optimization with hints and queries by \citet{pmlr-v202-bhaskara23a}. These works differ from ours in both the feedback structure and the nature of the side information: hints are typically vector-valued and correlated with the loss, whereas our queries return only ordinal information, namely the identity of the best arm in a single round.
A closely related query model is studied by \citet{bhaskara2023online}, where the learner may compare a small subset of actions and observe which action in the queried subset has the smallest loss before playing. Their results show that ordinal queries can lead to strong improvements in online learning and stochastic bandits. Our model is different in two key ways. First, a best-action query compares all arms rather than a small queried subset, and is therefore more informative at the single-round level. Second, and crucially, we impose a global query budget $k$, so the learner must decide how to allocate a scarce and expensive source of information over time.

\paragraph{Algorithms with predictions}
Our work also fits within the growing literature on algorithms with predictions, where machine-learned advice is used to improve performance beyond worst-case guarantees while retaining robustness when predictions are inaccurate \citep{mitzenmacher2022algorithms}. Examples include online algorithms for ski rental and scheduling \citep{NEURIPS2018_73a427ba}, caching with learned advice \citep{10.1145/3447579}, rent-or-buy problems with expert advice \citep{pmlr-v97-gollapudi19a}, and scheduling with learned weights \citep{DBLP:conf/soda/LattanziLMV20}. In contrast to much of this literature, our side information is not an exogenous prediction available at every round, but an actively requested oracle response subject to a hard budget. Thus, the main challenge is not only how to exploit side information when it is available, but also whether and when such scarce information can improve regret under partial feedback.

\paragraph{Bandits with preference feedback}
Our work is also connected to bandit models with ordinal, preference, or ranking feedback. 
In the dueling-bandit framework, the learner selects a pair of arms and observes only a noisy preference between them, rather than cardinal rewards \citep{10.1145/1553374.1553527, 10.1016/j.jcss.2011.12.028}. 
This line of work has led to algorithms based on pairwise preference estimates, such as Beat-the-Mean \citep{10.5555/3104482.3104513}, RUCB \citep{pmlr-v32-zoghi14}, and variants that target notions of optimality beyond Condorcet winners, such as Copeland winners \citep{NIPS2015_9872ed9f}. 
More general ranking-feedback models have also been studied in learning-to-rank and recommendation settings, including ranked bandits \citep{DBLP:conf/icml/RadlinskiKJ08}, cascading bandits \citep{pmlr-v37-kveton15}, top-$k$ feedback \citep{JMLR:v18:16-285}, and subset-wise ranking feedback under Plackett--Luce models \citep{pmlr-v98-saha19a}. 
Most closely related to this perspective, \citet{NEURIPS2024_936ce22b} introduce bandits with ranking feedback, where the learner receives ordinal information that allows arms to be ranked based on previous pulls, without observing numerical reward differences.
These works differ from ours in the type and timing of the ordinal information. 
In dueling-bandit and ranking-feedback models, ordinal feedback is typically the learner's primary feedback signal and is obtained repeatedly from queried pairs, subsets, ranked lists, or rankings induced by the history of past pulls. 
By contrast, in our setting the learner normally receives standard bandit rewards, and ordinal information is available only through a limited number of best-action queries.

\section{Multi-Armed Bandits With Best-Action Queries}\label{sec:preliminaries}

We consider a \emph{multi-armed bandit} (MAB) problem with $n \in \mathbb{N}$ arms and horizon $T \in \mathbb{N}$.
At each round $t \in [T]$,\footnote{Throughout the paper, we write $[a] := \{1, \ldots, a\}$ for the set of the first $a \in \mathbb{N}$ natural numbers.} the environment specifies a reward vector $X_t = (X_{t,1}, \dots, X_{t,n}) \in [0,1]^n$, the learner selects an arm (or action) $I_t \in [n]$, possibly at random, and observes the reward $X_{t,I_t}$ associated with the chosen arm. 
We study two different \emph{stochastic} settings. In the simpler \emph{i.i.d} setting, each reward $X_{t,i}$ is sampled  independently across rounds and arms from an distribution $D_i$. In the more general \emph{correlated} setting, the rewards of all the arms are sampled from a joint distribution $D$ at each round $t$, \emph{i.e.}, $X_t\sim D$. Notice that this is a special case of the \emph{adversarial} setting.

		\begin{protocol}[!htp]
			\caption{Learner-Environment Interaction}
			\label{alg:interaction}
			\begin{algorithmic}[1]
				\Require Number of available queries $k \in [T]$
				\For{$t = 1, \dots, T$}
				\State The environment selects $X_t = (X_{t,1}, \dots, X_{t,n}) \in [0,1]^n$
				\State The learner decides $Q_t \in \{0,1\}$ 
				\State $Q_t\gets Q_t \cdot \mathbb{I}\{\, \sum_{\tau=1}^t Q_\tau \le k \, \}$ 
				\If{$Q_t = 1$ }
				\State The oracle reveals $i_t^\star$
				\State The learner plays $I_t=i_t^\star$
				\Else
				\State The learner plays $I_t \in [n]$
				\EndIf
				\State The learner observes $X_{t, I_t}$
				\EndFor
			\end{algorithmic}
		\end{protocol}

In addition to the standard MAB interaction, the learner is endowed with a budget of $k \in [T]$ \emph{best-action queries}. At the beginning of each round $t \in [T]$, the learner may query the oracle. If a query is issued, the oracle reveals a reward-maximizing arm for that round, namely an arm $i_t^\star \in \arg\max_{i \in [n]} X_{t,i}$. In this case, the learner plays $I_t = i_t^\star$ and observes the corresponding reward $X_{t,i_t^\star}$.
Otherwise, the interaction proceeds as in a standard MAB: the learner selects an arm $I_t \in [n]$ and   then observes the reward $X_{t,I_t}$. The learner may issue at most $k$ queries in total over the $T$ rounds.
Formally, we let $Q_t \in \{0,1\}$ denote  whether a query is issued at round $t$, and define $\mathcal{Q} := \{t \in [T] \mid Q_t = 1\}$ as the set of rounds in which the learner queries the oracle. Notice that $|\mathcal{Q}|\le k$ must hold deterministically.
Protocol~\ref{alg:interaction} formalizes the interaction between the learner and the environment in a MAB problem with best-action queries.

The goal of the learner is to minimize the \emph{regret}. We denote the (cumulative) regret suffered by a learning algorithm $\mathfrak{A}$ as follows: 
\[
R_{T,k}(\mathfrak{A}) :=  \max_{i \in [n]} \,\, \sum_{t \in [T]} \E [X_{t,i}]  - \E \left[ \sum_{t \in [T]} X_{t,I_t} \right] \! \! ,
\]
where the expectation is taken with respect to the randomness of the algorithm and the environment.

\section{When Best-Action Queries Do Not Help: The Stochastic Correlated Setting}
\label{sec:adv_setting}

We start by providing an impossibility result for the case in which the rewards are stochastic and correlated across arms. This clearly implies that the same impossibility result holds for the more general adversarial setting. Specifically, we prove that, in contrast to the full-feedback setting~\citep{russo2024online}, best-action queries do not significantly improve the regret guarantees attainable, that is, any algorithm is guaranteed to suffer $\Omega(\sqrt{T-k})$ regret. Interestingly, this shows that correlation is sufficient to make best-action queries essentially useless.

Our impossibility result is based on building a family of two hard instances $\nu_1$, $\nu_2$. In both instances, we use a uniform random variable $W_t\sim \text{Unif}\{1,2\}$ to correlate the rewards of the two possible arms. Then, the two instances $\nu_1$, $\nu_2$ are defined as follows:
\[
\circled{\nu_1} \,\,: \,\, (X_{t,1},X_{t,2})=
\begin{cases}
	(Z_t^{+},\,Y_t), & \text{if } W_t=1,\\[2mm]
	(Z_t^{-},\,Z_t^{+}), & \text{if } W_t=2.
\end{cases}
\]
\[
\circled{\nu_2} \,\,: \,\, (X_{t,1},X_{t,2})=
\begin{cases}
	(Z_t^{+},\,Z_t^{-}), & \text{if } W_t=1,\\[2mm]
	(Y_t,\,Z_t^{+}), & \text{if } W_t=2.
\end{cases}
\]
In both instances, $Y_t\sim \text{Unif}[0,1]$ and $Z_t^{-}$, $Z_t^{+}$ are chosen as a function of $Y_t$ in order to guarantee that $Y_t <Z_t^{-} < Z_t^{+}$ holds almost surely (we defer the formal construction to the proof).

It is easy to see that the two instances differ only in which arm gets the smaller reward $Z_t^{-}$. Indeed, under $\nu_1$ this creates a slight advantage for arm $1$ in expectation, whereas under $\nu_2$ the same advantage is given to arm $2$. In the following, we denote by \( \Delta \) the gap between the expected rewards of the optimal arm and the suboptimal arms.

Our construction deals with the tension between two opposing objectives: keeping the KL between the two instances small and control the negative regret obtained by best-action queries.
The main challenge is dealing with the tradeoff between the suffered (positive) regret and the negative regret accumulated by using best-action queries. To ensure that the (negative) regret is small in magnitude, we would like the best fixed arm, \emph{i.e.}, our baseline \( i^\star \), to be optimal at every round. This idea has been used in previous work such as~\citep{gerchinovitz2016refined}. However, this is not suitable in our setting, since it would allow the learner to identify the optimal arm with a single query.

We use correlated rewards to deal with this tension.
Indeed, in our construction a query reveals the arm that is optimal at the current round and allows the learner to collect the top reward \( Z_t^{+} \), but this feedback has the same distribution under \( \nu_1 \) and \( \nu_2 \), and therefore does not help identify which arm is optimal in expectation. In contrast, on rounds without queries, the learner observes only the reward of the selected arm, and these observations are drawn from two nearby distributions whose KL divergence is of order \( \Delta^2 \). Hence, after \( m := T - k \) rounds without queries, the total information available for distinguishing \( \nu_1 \) from \( \nu_2 \) is only of order \( m \Delta^2 \). Choosing \( a = \Theta(m^{-1/2}) \) keeps the two environments statistically close, so any learner must frequently pull the wrong fixed arm on non-queried rounds, yielding regret of order \( m \cdot a = \Theta(\sqrt{T - k}) \).

Based on the above reasoning, we can now state the following lower bound.
\begin{restatable}{theorem}{advlower}\label{thm:lb_correlated}
	For every $T \ge 1$ and $k \in [T-1]$, and every randomized learning algorithm $\mathfrak{A}$ that uses $k$ best-action queries, there exists a MAB problem instance with two arms and stochastic correlated rewards such that the regret of $\mathfrak{A}$ satisfies:
	\[
	R_{T,k}(\mathfrak{A}) \ge \Omega(\sqrt{T-k}).
	\]
\end{restatable}

We finally remark that any adversarial no-regret algorithm which attains a regret upper bound $\widetilde{\mathcal{O}}({\sqrt{T}})$ (\emph{e.g.},~\citep{auer2002nonstochastic}) provides $\widetilde{\mathcal{O}}({\sqrt{T-k}})$ regret when endowed with best-action queries. This is because the regret incurred by any algorithm during the best-action query rounds is at most zero.

\section{Best-Action Queries in the Stochastic i.i.d.\ Setting}
\label{sec: stochastic}
In this section, we show that when the rewards are i.i.d.\ across both rounds and arms, the benefit of best-action queries improves dramatically. Specifically, the results of this section will be presented in the following order:
\begin{itemize}[topsep=0pt, leftmargin=*]
	\item In Section~\ref{sec:var}, we study the gap between the value of the optimal static policy and the optimal dynamic one, highlighting a clear connection between this gap and the variance of the instance.
	\item In Section~\ref{sec:UCB}, we present our algorithm and the main intuition behind it.
	\item In Section~\ref{sec:final}, we provide the theoretical guarantees achieved by our algorithm.
	\item In Section~\ref{sec:stoc_lower}, we prove a matching lower bound, showing that the regret guarantee achieved by our algorithm is optimal.
\end{itemize}

\subsection{The Gap Between Dynamic and Static Policies and Its Relation to Variance}\label{sec:var}

As a first step, in this section we establish a relationship between the gap separating the optimal dynamic and static policies and the variance of the rewards. To this end, we first need to introduce some additional useful notation for the stochastic i.i.d.\ setting.

For ease of presentation, we let $X_i \sim {D}_i$ denote a generic random variable drawn from the reward distribution ${D}_i$ of arm $i \in [n]$.
Then, we define $\mu_i := \E[X_i]$ and $\sigma_i^2 := \Var(X_i)$ as the mean and variance of arm $i$'s reward distribution.
Additionally, we denote with $i^\star\in \arg\max_{i \in [n]} \mu_i$ an optimal arm and with $\mu^\star= \mu_{i^\star}$ its mean reward. 
As usual, we denote by $\Delta_i := \mu^\star - \mu_i$ the suboptimality gap of arm $i\in [n]$ and we denote with $\mathcal S := \{ i \in [n] \mid \Delta_i > 0 \}$ the set of suboptimal arms.
We also define two quantities corresponding to the per-round optimal value of the optimal \emph{static} and \emph{dynamic} policies, respectively, as follows:
\[
\OPTc :=  \max_{i \in [n]} \E\left[ X_i\right] = \mu^\star,
\qquad
\OPTd := \E\!\left[\max_{i \in [n]} X_i\right] \!.
\]
Intuitively, the first policy just knows the distribution, while the second one picks an arm \emph{after} observing a sample from the distribution.
Now, we show how to relate the difference between the dynamic and static optimal values to the suboptimality gaps and the variances of the arms. 
To do that, we observe that:
\begin{equation*}
	\OPTd-\OPTc
	= \E\!\left[\max_{j \in [n]} X_j\right]-\mu_{i^\star}
	\ge \frac{1}{n}\sum_{i \in [n]} \Big(\E[\max\{X_{i^\star},X_i\}]-\mu_{i^\star}\Big).
\end{equation*}
The right-hand side of the inequality is related to the variance of the variables. Specifically,
\begin{equation*}
	\E[\max\{X_{i^\star},X_i\}]-\mu_{i^\star}
	= \frac{1}{2}\left(\E[|X_{i^\star}-X_i|]-\Delta_i \right) \quad \forall i \in S,
\end{equation*}
which follows from employing $\max\{a,b\}=(a+b+|a-b|)/2$.
As a final observation, notice that:
\begin{equation*}
	\E\left[\left|X_{i^\star}-X_i\right|\right]
	\ge \E\left[(X_{i^\star}-X_i)^2\right]
	= \Var\left(X_{i^\star}-X_i\right)+\left(\E\left[X_{i^\star}-X_i\right]\right)^2
	\ge \sigma_{i^\star}^2+\sigma_i^2 \ge \sigma_i^2,
\end{equation*}
since the rewards are independent. Define $(\,\cdot\,)_{+}\coloneqq \max\{\,\cdot\,,0\}$.
We can now combine these results to get the following lemma.
\begin{restatable}{lemma}{DeltaOPT}\label{lem:DeltaOpt-iid}
	Given a MAB instance with stochastic i.i.d.\ rewards, the following holds:
	\[
	\OPTd - \OPTc \ge \frac{1}{2n}\sum_{i \in S} \left(\sigma_i^2 - \Delta_i\right)_{+}.
	\]
\end{restatable}
Lemma~\ref{lem:DeltaOpt-iid} relates the difference between optimal \emph{dynamic} and \emph{static} policies, showing that this gap admits a lower bound that grows with the variance of suboptimal arms whenever the latter is sufficiently large compared to their suboptimality gap.

We now show how the gap between the dynamic and static optimal values affects the learner’s regret. Intuitively, whenever the learner queries the oracle, it receives the dynamic optimal value, while the benchmark against which regret is measured—namely, the static optimum—remains unchanged. Therefore, each query contributes exactly the gap between the dynamic and static optimal values, which yields the following decomposition of the learner’s regret.
\begin{restatable}{lemma}{KOptimalQueries}\label{lem:DeltaOpt}
	Consider any algorithm $\mathfrak{A}$ that, given as input a time horizon $T \in \mathbb{N}$ and a number $k \in [T]$ of allowed best-action queries, issues a query to the oracle for exactly $k$ rounds and, in the remaining $T-k$ rounds, ignores the oracle feedback and employs an algorithm $\mathfrak{A}_0$ for standard stochastic i.i.d. MAB problems. Then, the following holds:
	\[
	R_{T,k}(\mathfrak{A}) = R_{T-k,0}(\mathfrak{A}_0) - k(\OPTd-\OPTc).
	\]
\end{restatable}
\begin{proof}
	Let $Q \subseteq [T]$ be the set of query rounds. By assumption, $|Q|=k$.
	On every queried round $t \in Q$, the learner plays an arm:
	\[
	i_t^\star \in \arg\max_{i \in [n]} X_{t,i},
	\]
	and thus the reward is
	$X_{t,I_t} = \max_{i \in [n]} X_{t,i}.$
	Since the rewards are i.i.d., for every queried round $t\in\mathcal{Q}$ we have
	$
	\mathbb E[X_{t,I_t}] = \mathbb E\!\left[\max_{i \in [n]} X_i\right] = \mathrm{OPT}_d.
	$
	Therefore,
	\begin{align}\label{eq:query1}
		\mathbb E\!\left[\sum_{t \in Q} X_{t,I_t}\right] = k\,\mathrm{OPT}_d.
	\end{align}
	
	Now consider the $T-k$  rounds without queries. By assumption, on those rounds the algorithm ignores the oracle feedback and runs the standard stochastic bandit algorithm $\mathfrak{A}_0$. Since rewards are i.i.d. across rounds, the sequence of observations collected on the non-query rounds has the same law as the sequence of observations that $\mathfrak{A}_0$ would receive in a fresh stochastic bandit instance of horizon $T-k$. Thus, we have:
	\begin{align}\label{eq:query2}
		\mathbb E\!\left[\sum_{t \notin Q} X_{t,I_t}\right]
		= (T-k)\,\mathrm{OPT}_s - R_{T-k,0}(\mathfrak{A}_0).
	\end{align}
	
	Summing Equation~\eqref{eq:query1} and Equation~\eqref{eq:query2} we get
	$
	\mathbb E\!\left[\sum_{t=1}^T X_{t,I_t}\right]
	= (T-k)\,\mathrm{OPT}_s - R_{T-k,0}(\mathfrak{A}_0) + k\,\mathrm{OPT}_d.
	$
	Thus,
	\begin{align*}
		R_{T,k}(\mathfrak{A})
		&= T\,\OPTc  - \mathbb E\!\left[\sum_{t=1}^T X_{t,I_t}\right] \\
		&= T\,\OPTc - \Big((T-k)\,\OPTc  - R_{T-k,0}(\mathfrak{A}_0) + k\,\mathrm{OPT}_d\Big) \\
		&= R_{T-k,0}(\mathfrak{A}_0) - k(\mathrm{OPT}_d-\mathrm{OPT}_s).
	\end{align*}
	This concludes the proof.
\end{proof}
Lemma~\ref{lem:DeltaOpt} shows that endowing an algorithm with $k$ best-action queries reduces its regret by an amount proportional to both $k$ and the gap $\OPTd-\OPTc$.
We can now combine Lemma~\ref{lem:DeltaOpt-iid} and Lemma~\ref{lem:DeltaOpt}. This allows to lower bound the gap between dynamic and static optimal values $\OPTd-\OPTc$, in order to get the following corollary.
\begin{corollary}
	For every $k \in [T]$, under the same assumptions of Lemma~\ref{lem:DeltaOpt}, it holds:
	\[
	R_{T,k}(\mathfrak{A}) \le R_{T-k,0}(\mathfrak{A}_0) - \frac{k}{2n}\sum_{i \in S} (\sigma_i^2 - \Delta_i)_+.
	\]
\end{corollary}
The above corollary has an important implication. Consider hard instances in which the regret of a standard stochastic MAB algorithm is of order $\Theta(\sqrt{T})$. In such instances, the arm variances remain constant, while the suboptimality gaps satisfy $\Delta_i = \Theta(1/\sqrt{T})$. Hence, for sufficiently large $T$, each term $(\sigma_i^2 - \Delta_i)_+$ is bounded away from zero, and therefore the second term in the corollary is of order $\Theta(k)$. 
In other words, in these instances, employing $k$ best-action leads to a $\Theta(k)$ improvement in the final regret guarantees.
For instance, if we set $k = T^{2/3}$ and choose UCB1~\citep{auer2002finite} as the underlying algorithm, then the negative term induced by the queries eventually dominates the usual regret term of UCB1, so that the resulting regret becomes negative for sufficiently large $T$. This suggests that access to best-action queries can dramatically improve performance.

\subsection{A Variance-Aware Regret Minimizer}
\label{sec:UCB}

In the previous section, we showed how best-action queries can significantly improve the regret guarantees when the variance-related quantities $(\sigma_i^2-\Delta_i)_+$ are large. In this section, we deal with the settings in which such quantities are small, and best-action queries are therefore \emph{not} helpful. 
In this case, we show that using a variance-aware regret minimizer in the rounds without queries leads to a small regret. Accordingly, 
we employ the UCB-V algorithm~\citep{audibert2009exploration} as a regret minimizer. For the sake of clarity, we restate here the regret guarantees of UCB-V.
\begin{theorem}[UCB-V guarantees~\citep{audibert2009exploration}]\label{thm:UCB-V}
	Consider a MAB instance with stochastic i.i.d.\ rewards in $[0,1]$, and let $\mathcal S := \{ i \in [n] \mid \Delta_i > 0 \}$
	denote the set of strictly suboptimal arms. For every horizon \( T \ge 1 \), let \( N_i(T) \) denote the number of times arm \( i \in [n] \) has been selected up to round \( T \).
	Then, for every
	$i \in \mathcal{S}$, the UCB-V algorithm guarantees:
	\[
	\Delta_i \E[N_i(T)]
	\le
	C\left(\frac{\sigma_i^2}{\Delta_i}+1\right)\log T,
	\]
	for an absolute constant $C>0$.
\end{theorem}

		\begin{algorithm}[!htp]
			\caption{UCB-V with Best-Action Queries}
			\label{alg:ucbv-probing}
			\begin{algorithmic}[1]
				\Require Number of available queries $k \in [T]$
				\For{$t = 1, \dots, k$}
				\State The oracle reveals $i_t^\star$
				\State The learner plays $I_t=i_t^\star$
				\EndFor
				\State Initialize UCB-V from scratch 
				\State Run UCB-V for the remaining $T-k$ rounds
			\end{algorithmic}
		\end{algorithm}

Intuitively, Theorem~\ref{thm:UCB-V} shows how, by employing the UCB-V algorithm, it is possible to attain regret guarantees that strongly benefit from instances characterized by low variance. For instance, when $\sigma_i^2=\Theta(\Delta_i)$, a regret upper bound of order $\log T$ can be recovered even for $\Delta_i=\Theta(1/\sqrt{T})$. 

In Algorithm~\ref{alg:ucbv-probing}, we show how to instantiate UCB-V into our setting in which the learner is endowed with best-action queries. 
Specifically, for the first $k$ rounds the learner exploits the queries to the oracle. Notice that, since the rewards are stochastic, spending the entire query budget at the beginning of the learning process yields the same expected reward as any other way of allocating the same budget over time. Thus, for simplicity, we deplete the query budget in the first rounds. Then, UCB-V is initialized from scratch, \emph{i.e.}, we discard all the information collected during the probing phase, since they are not needed to improve the regret guarantees of our algorithm. 
Finally, we run the UCB-V algorithm for the remaining $T-k$ rounds. 

\subsection{The Final Regret Bound}\label{sec:final}

In this section, we show how to combine the results of \Cref{sec:var} and \Cref{sec:UCB} to deal with the tension between the variance and the gap $\OPTd-\OPTc$. Specifically, we show that when the variance is large then $\OPTd-\OPTc$ is large, making the queries extremely useful. On the other hand, when the variance is small, a variance-aware regret minimizer as UCB-V provides a much better regret rate the the standard $\widetilde{\mathcal{O}}(\sqrt{T})$.

Formally, we prove the following theorem.
\begin{restatable}{theorem}{ThmUpperBound}\label{thm:upper}
	Fix $k \in [T]$ and let $\mathfrak{A}$ be Algorithm~\ref{alg:ucbv-probing} with input $k$. Then,
	\[
	R_{T,k}(\mathfrak{A}) \le   \mathcal{O}\left( \min \left \{\frac{nT \log T}{k}, \sqrt{nT\log T} \right \} \right).
	\]
\end{restatable}
\begin{proof}
	The claim is trivial when $k=T$, since the learner queries in every round and thus: 
	$R_{T,T}(\mathfrak{A})=T \cdot (\,\OPTc -\,\OPTd ) \le 0.$
	Hence, in the rest of the proof, we assume $k \le T-1$.
	We define $M := T-k$. For each $i\in \mathcal{S}$, let $N_i(M)$ denote the number
	of times arm $i \in [n]$ is pulled by the UCB-V algorithm after the initial querying phase.
	By Lemma~\ref{lem:DeltaOpt},
	\[
	R_{T,k}(\mathfrak{A})=R_{M,0}(\mathrm{\text{UCB\text{-}V}})+k(\OPTc-\OPTd).
	\]
	Since only strictly suboptimal arms contribute to the regret of the UCB-V
	phase, it holds:
	\[
	R_{M,0}({\text{UCB\text{-}V}})
	=
	\sum_{i\in\mathcal S}\Delta_i \E[N_i(M)],
	\]
	where $\E[N_i(M)]$ is the expected number of times UCB-V played arm $i$ in expectation.
	
	Notice that we can employ Lemma~\ref{lem:DeltaOpt} since our algorithm performs best-action queries in the first $k$ rounds and then runs the UCB-V algorithm after these rounds. Moreover, by Lemma~\ref{lem:DeltaOpt-iid},
	\[
	\OPTd-\OPTc
	\ge
	\frac{1}{2n}\sum_{i\in\mathcal S} (\sigma_i^2-\Delta_i)_+.
	\]
	Therefore, it holds:
	\begin{equation}\label{eq:regret_decomposition}
		R_{T,k}(\mathfrak{A})
		\le
		\sum_{i\in\mathcal S}
		\left[
		\Delta_i \E[N_i(M)]
		-
		\frac{k}{2n}(\sigma_i^2-\Delta_i)_+
		\right]\!.
	\end{equation}

	For each $i\in \mathcal{S}$, we define $R_i :=\Delta_i \E[N_i(M)] $.
	By \Cref{thm:UCB-V}, we have that:
	\[R_i\le C \left(\frac{\sigma_i^2}{\Delta_i}+1\right)\log M =  2 C \log M +  C\left(\frac{\sigma_i^2}{\Delta_i}-1\right)\log M . \]
	Let $x_i := (\sigma_i^2-\Delta_i)_+$, then we have:
	\begin{align*}
		R_i&= \min \left \{ 2 C \log M +  C\left(\frac{\sigma_i^2}{\Delta_i}-1\right)\log M, \Delta_i \E[N_i(M)] \right \}\\
		&\le 2 C \log M + \min \left \{   C\left(\frac{\sigma_i^2}{\Delta_i}-1\right)\log M, \Delta_i \E[N_i(M)] \right  \}\\
		& \le 2 C \log M + \sqrt{ C \left(\frac{\sigma_i^2}{\Delta_i}-1\right)_{+}\Delta_i \E[N_i(M)] \log M }\\
		& \le 2 C \log T + \sqrt{ C x_i \E[N_i(M)]\log T } ,
	\end{align*}
	where in the second inequality we used $\min\{a,b\}\le \sqrt{a b_{+}}$ that holds for every $a> 0$.

	Then, 
	\begin{align}\label{eq:Ri_k2n}
		R_i - \frac{k}{2n}x_i \le 2 C \log T + \sqrt{ C x_i\E[N_i(M)] \log T  } - \frac{k}{2n}x_i.
	\end{align}
	Now apply the inequality
	$
	a\sqrt{x} - bx \le {a^2}/{4b}
	$
	that holds for every $a,b > 0$ and $x \ge 0$, where we choose $a := \sqrt{C\E[N_i(M)]\log T}$ and $b := {k}/{2n}$, which gives the following bound:
	\begin{equation}\label{eq:Ri_k2n_2}
		\sqrt{C\,x_i\,\E[N_i(M)]\,\log T}
		-
		\frac{k}{2n}x_i
		\le
		\frac{Cn}{2k}\E[N_i(M)]\log T.
	\end{equation}
	
	Therefore, plugging Equation~\eqref{eq:Ri_k2n_2} into Equation~\eqref{eq:Ri_k2n}, for every $i \in \mathcal{S}$, we get:
	\[
	R_i-\frac{k}{2n}x_i
	\le
	2C\log T+\frac{Cn}{2k}\E[N_i(M)]\log T.
	\]
	Summing over all $i\in \mathcal{S}$, and recalling Equation~\eqref{eq:regret_decomposition} and the definitions of $R_i$ and $x_i$, we obtain:
	\[
	R_{T,k}(\mathfrak{A})
	\le
	2Cn\log T
	+
	\frac{Cn\log T}{2k}\sum_{i\in \mathcal{S}}\E[N_i(M)].
	\]
	The number of pulls of suboptimal arms during the UCB-V phase is at most the total number of pulls, \emph{i.e}, $\sum_{i\in\mathcal S}\E[N_i(M)] \le M \le T$. Thus, it holds:
	\[
	R_{T,k}(\mathfrak{A})
	\le
	2Cn\log T+\frac{CnT\log T}{2k} \le \frac{5C}{2}\frac{nT\log T}{k},
	\]
	since $k\le T$. To conclude the proof, we observe that UCB-V attains the same instance-independent regret upper bound as UCB1~\citep{auer2002finite}, which is of order $\mathcal{O}(\sqrt{nM\log M})$.
\end{proof}

\subsection{A Matching Lower bound}
\label{sec:stoc_lower}

In this section, we provide a matching lower bound for the stochastic i.i.d. setting. 
For clarity, we describe the core intuition using a two-arm setting, though our theorem and formal proof generalizes to an arbitrary number of arms $k$.
Our construction stems from the observation that the hardest instances arise when the problem parameters satisfy the following relationship:
\[  \sigma^2 =\Theta \left(\Delta \frac{T}{k} \right)\quad \text{and} \quad  \Delta = \Theta\left( \frac{1}{k}\right).\]
We consider two problem instances, $\nu_1$ and $\nu_2$. In $\nu_1$, the first arm follows a Bernoulli distribution with mean $1-p$, and the second arm follows a Bernoulli with mean $1-p+\Delta$. In $\nu_2$, these reward distributions are swapped.
A distinctive feature of our approach is the use of Bernoulli distributions where the mean $p \simeq 1$. While standard lower bounds often use distributions centered around $1/2$, as we will see in the following, our construction leverages the fact that the mean approaches the boundary. In principle, this makes the instance easier to learn, since the variance is smaller and the KL divergence scales as $\mathcal O(1/\Delta)$ rather than the typical $\mathcal O(1/\Delta^2)$.
Nonetheless, the motivation behind this choice is rooted in Lemma \ref{lem:DeltaOpt-iid}, which shows that high-variance distributions (like those centered at $1/2$) incur large negative regret during the querying rounds.
To see this, consider two Bernoulli distributions with means close to $1/2$. In this case, the gap between the value of optimal dynamic policy $\OPTd \simeq 3/4$ and the optimal static one $\OPTc = 1/2$ is constant. However, by shifting the distributions such that means are at least $1-p$, this gap is compressed to at most $p$, as $\OPTc = 1-p$ and $\OPTd \le 1$.
This compression is essential for matching the upper bound.

Formally we prove the following.
\begin{restatable}{theorem}{LBDue}\label{thm:lower_stoc}
	There exists an absolute constant $C$ such that, for any {randomized} algorithm $\mathfrak{A}$ which issues $k$ best-action queries with $k\leq \nicefrac{T}{C}$, it holds:
	\[  R_{T,k}(\mathfrak{A}) \ge \Omega\left(\min\left\{\frac{Tn}{k}, \sqrt{n(T-k)}\right\}\right).  \]
\end{restatable}
Theorem~\ref{thm:lower_stoc} shows that Algorithm~\ref{alg:ucbv-probing} is optimal up to logarithmic factors in the stochastic i.i.d. setting.

\bibliographystyle{plainnat}

\bibliography{example_paper.bib}

\newpage
\appendix

\section{Omitted Proofs of Section~\ref{sec:adv_setting}}

\advlower*
\begin{proof}
	We use Yao's minimax principle. Hence, it is enough to fix an arbitrary deterministic algorithm $\mathfrak{A}$
	and construct a distribution over instances on which its expected regret is
	$\Omega(\sqrt{T-k})$.
	Let $m := T-k$ be the number of rounds without queries.
	Fix parameters
	\[
	0<a<\frac14,
	\qquad
	0<\eta<\frac14-a,
	\qquad
	b:=a+\eta.
	\]
	For every $c \in [0,1/2]$, we define
	\[
	H_c(x):=x-cx(1-x), \qquad x\in[0,1].
	\]
	Since
	\[
	H_c'(x)=1+c(2x-1)\ge \frac12,
	\]
	the map $H_c$ is strictly increasing and therefore invertible on $[0,1]$.
	Let $(U_t)_{t=1}^T$ be i.i.d.\ $\mathrm{Unif}[0,1]$, let $(W_t)_{t=1}^T$ be i.i.d.\
	$\mathrm{Unif}\{1,2\}$, and define
	\[
	Y_t:=U_t,
	\qquad
	Z_t^-:=H_a^{-1}(U_t),
	\qquad
	Z_t^+:=H_b^{-1}(U_t).
	\]
	Because $a<b$, we have $H_b(x)\le H_a(x)\le x $ for all $ x\in[0,1],$	and therefore $Y_t \le Z_t^- \le Z_t^+.$
	We now define two stochastic correlated instances $\nu_1,\nu_2$:
	\[
	\nu_1:\quad
	(X_{t,1},X_{t,2})=
	\begin{cases}
		(Z_t^+,Y_t), & \text{if } W_t=1,\\
		(Z_t^-,Z_t^+), & \text{if } W_t=2,
	\end{cases}
	\]
	and
	\[
	\nu_2:\quad
	(X_{t,1},X_{t,2})=
	\begin{cases}
		(Z_t^+,Z_t^-), & \text{if } W_t=1,\\
		(Y_t,Z_t^+), & \text{if } W_t=2.
	\end{cases}
	\]
	In both instances, the arm indexed by $W_t$ is the unique best arm at round $t$.
	For every $c\in[0,1/2]$ and every $y\in[0,1]$,
	\[
	\mathbb{P}\!\left(H_c^{-1}(U_t)\le y\right)
	=
	\mathbb{P}\!\left(U_t\le H_c(y)\right)
	=
	H_c(y),
	\]
	hence $H_c^{-1}(U_t)$ has density
	\[
	f_c(x)=H_c'(x)=1+c(2x-1), \qquad x\in[0,1].
	\]
	Therefore,
	\[
	\mathbb{E}[H_c^{-1}(U_t)]
	=
	\int_0^1 x\bigl(1+c(2x-1)\bigr)\,dx
	=
	\frac12+\frac{c}{6}.
	\]
	In particular,
	$
	\mathbb{E}[Y_t]=1/2,
	\,
	\mathbb{E}[Z_t^-]=1/2+{a}/{6},
	\,
	\mathbb{E}[Z_t^+]=1/2+{b}/{6}.
	$
	Under $\nu_1$, arm $1$ has mean
	\[
	\mu_1^{(1)}
	=
	\frac12\mathbb{E}[Z_t^+]+\frac12\mathbb{E}[Z_t^-]
	=
	\frac12+\frac{a+b}{12},
	\]
	whereas arm $2$ has mean
	\[
	\mu_2^{(1)}
	=
	\frac12\mathbb{E}[Y_t]+\frac12\mathbb{E}[Z_t^+]
	=
	\frac12+\frac{b}{12}.
	\]
	Thus, under $\nu_1$, the best fixed arm is arm $1$, with gap $\Delta=\mu_1^{(1)}-\mu_2^{(1)}=\frac{a}{12}.$
	By symmetry, under $\nu_2$ the best fixed arm is arm $2$, with the same gap $\Delta=a/12$.
	
	Let $R_T^{(1)}(\mathfrak{A})$ and $R_T^{(2)}(\mathfrak{A})$ be the expected regrets of $\mathfrak{A}$ under $\nu_1$ and $\nu_2$,
	respectively.\footnote{We drop the dependence on $k$ from the definition of regret to ease the notation.}
	On a queried round, the oracle reveals the best arm index $W_t$, so the learner obtains reward
	$Z_t^+$. Since the best fixed arm has mean $1/2+(a+b)/12$, the expected regret of a queried round is
	\[
	\left(\frac12+\frac{a+b}{12}\right)-\left(\frac12+\frac{b}{6}\right)
	=
	-\frac{b-a}{12}
	=
	-\frac{\eta}{12}.
	\]
	Moreover, the queried-round feedback is exactly $(W_t,Z_t^+)$ under both $\nu_1$ and $\nu_2$, so
	queried rounds give no information for distinguishing the two instances.
	
	Let $N_1$ and $N_2$ be the numbers of times the learner pulls arms $1$ and $2$, respectively, over
	the $m=T-k$ non-queried rounds. Then $N_1+N_2=m$. On a non-queried round, pulling the wrong
	fixed arm incurs expected regret exactly $\Delta=a/12$. Hence
	\[
	R_T^{(1)}(\mathfrak{A})\ge \Delta\,\mathbb{E}_1[N_2]-\frac{\eta k}{12},
	\qquad
	R_T^{(2)}(\mathfrak{A})\ge \Delta\,\mathbb{E}_2[N_1]-\frac{\eta k}{12},
	\]
	where $\mathbb{E}_\sigma$ denotes expectation under $\nu_\sigma$.
	Now define the event $E:=\{N_1>m/2\}.$
	Since $E^c$ implies $N_2\ge m/2$, we obtain
	\[
	\mathbb{E}_1[N_2]\ge \frac{m}{2}\mathbb{P}_1(E^c),
	\qquad
	\mathbb{E}_2[N_1]\ge \frac{m}{2}\mathbb{P}_2(E).
	\]
	Therefore,
	\[
	R_T^{(1)}(\mathfrak{A})+R_T^{(2)}(\mathfrak{A})
	\ge
	\frac{\Delta m}{2}\Bigl(\mathbb{P}_1(E^c)+\mathbb{P}_2(E)\Bigr)-\frac{\eta k}{6}.
	\]
	By the Bretagnolle--Huber inequality,
	\[
	\mathbb{P}_1(E^c)+\mathbb{P}_2(E)
	\ge
	\frac12 \exp\bigl(-D_{\mathrm{KL}}(\mathbb{P}_1\|\mathbb{P}_2)\bigr),
	\]
	where $\mathbb{P}_1,\mathbb{P}_2$ are the laws of the feedback histories under $\nu_1,\nu_2$. Hence
	\begin{equation}\label{eq:regret_sum}
		R_T^{(1)}(\mathfrak{A})+R_T^{(2)}(\mathfrak{A})
		\ge
		\frac{\Delta m}{4}\exp \bigl(-D_{\mathrm{KL}}(\mathbb{P}_1\|\mathbb{P}_2)\bigr)-\frac{\eta k}{6}.
	\end{equation}
	It remains to bound the KL divergence. On a non-queried round, if the learner plays arm $1$, then
	the observation law is
	\[
	p_+(x)=\frac12 f_b(x)+\frac12 f_a(x)
	=
	1+\frac{a+b}{2}(2x-1)
	\]
	under $\nu_1$, and
	\[
	p_-(x)=\frac12\cdot 1+\frac12 f_b(x)
	=
	1+\frac{b}{2}(2x-1)
	\]
	under $\nu_2$. If the learner plays arm $2$, the two laws are swapped. Since $a+b<1/2$, both
	densities are uniformly lower bounded by $3/4$ on $[0,1]$. Moreover,
	\[
	p_+(x)-p_-(x)=a\left(x-\frac12\right).
	\]
	Using $\log u\le u-1$, we get
	\[
	D_{\mathrm{KL}}(p_+\|p_-)
	\le
	\int_0^1 \frac{(p_+(x)-p_-(x))^2}{p_-(x)}\,dx
	\le
	\frac43 a^2 \int_0^1 \left(x-\frac12\right)^2 dx
	=
	\frac{a^2}{9}.
	\]
	By the same argument,
	\[
	D_{\mathrm{KL}}(p_-\|p_+)\le \frac{a^2}{9}.
	\]
	Since queried rounds contribute zero KL and there are exactly $m$ non-queried rounds, the chain rule
	gives
	\begin{equation}\label{eq:kl_decomposition}
		D_{\mathrm{KL}}(\mathbb{P}_1\|\mathbb{P}_2)\le m\max\{D_{\mathrm{KL}}(p_+\|p_-),D_{\mathrm{KL}}(p_-\|p_+)\}
		\le \frac{ma^2}{9}.
	\end{equation}
	Now, we define:
	\[
	a:=\min\left\{\frac18,\frac{1}{\sqrt m}\right\},
	\qquad
	\eta:=\frac{a}{24(k+1)},
	\qquad
	b:=a+\eta.
	\]
	Then $0<a<1/4$, $\eta>0$, and $a+b=2a+\eta<1/2,$ so all previous bounds apply. Also,
	\[
	\frac{ma^2}{9}\le \frac19,
	\qquad
	\frac{\eta k}{6}
	=
	\frac{ak}{144(k+1)}
	\le
	\frac{a}{144}.
	\]
	Plugging these estimates and $\Delta = a/12$ into Equation~\eqref{eq:regret_sum}, we obtain:
	\begin{equation} \label{eq:regret_sum_lb}
		R^{(1)}_T(\mathfrak{A}) + R^{(2)}_T(\mathfrak{A})
		\ge \frac{am}{48} e^{-1/9} - \frac{a}{144}.
		\tag{3}
	\end{equation}
	Now recall that $a = \min\{1/8,\,1/\sqrt{m}\}$. For every $m \ge 1$, we have
	$
	am \ge {\sqrt{m}}/{8}
	\, \textnormal{and} \,
	a \le {1}/{8}.
	$
	Therefore, using Equation~\eqref{eq:regret_sum_lb}, we have:
	\[
	R^{(1)}_T(\mathfrak{A}) + R^{(2)}_T(\mathfrak{A})
	\ge \frac{e^{-1/9}}{384}\sqrt{m} - \frac{1}{1152}.
	\]
	Since ${m} \ge 1$ and  $e^{-1/9} \ge 8/9$, it follows that
	\[
	R^{(1)}_T(\mathfrak{A}) + R^{(2)}_T(\mathfrak{A})
	\ge \left(\frac{1}{432} - \frac{1}{1152}\right)\sqrt{m}
	= \frac{5}{3456}\sqrt{m}.
	\]
	Recalling that $m = T-k$, we conclude that
	\[
	R^{(1)}_T(\mathfrak{A}) + R^{(2)}_T(\mathfrak{A}) \ge \frac{5}{3456}\sqrt{T-k}.
	\]
	Therefore, under the uniform prior on $\{\nu_1,\nu_2\}$,
	\[
	\frac{R^{(1)}_T(\mathfrak{A}) + R^{(2)}_T(\mathfrak{A})}{2}
	\ge \frac{5}{6912}\sqrt{T-k}.
	\]
	By Yao's minimax principle, there exists a stochastic correlated two-armed instance on which
	every randomized learner with $k$ best-action queries suffers regret
	\[
	R_{T,k}(\mathfrak{A}) \ge \Omega(\sqrt{T-k}),
	\]
	as claimed.
\end{proof}

\section{Omitted Proofs of Section~\ref{sec: stochastic}}
\DeltaOPT*
\begin{proof}
	%
	Since
	$
	\max_{j \in [n]} X_j \ge \max\{X_{i^\star}, X_i\}$ for every $ i \in [n], $ averaging over $i\in[n]$ yields
	\begin{equation}\label{eq:deltaX}
		\max_{j \in [n]} X_j \ge \frac{1}{n}\sum_{i=1}^n \max\{X_{i^\star}, X_i\}.
	\end{equation}
	Then, we have:
	\begin{align}
		\OPTd - \OPTc
		=
		\E\!\left[\max_{j \in [n]} X_j\right] - \mu_{i^\star}
		&\ge \frac{1}{n}\sum_{i=1}^n \Bigl(\E[\max\{X_{i^\star}, X_i\}] - \mu_{i^\star}\Bigr) \nonumber \\
		&=\frac{1}{n}\sum_{i\neq i^\star} \Bigl(\E[\max\{X_{i^\star}, X_i\}] - \mu_{i^\star}\Bigr)  \nonumber \\
		&\ge\frac{1}{2n}\sum_{i\neq i^\star} \Bigl( \E[|X_{i^\star} - X_i|] - \Delta_i  \Bigr)_{+}. \label{eq:DeltaOPT_lb}
	\end{align}
	The first inequality above holds by taking expectations and subtracting $\mu_{i^\star}$ from both sides of Equation~\eqref{eq:deltaX}.
	The first equality holds since $\E[\max\{X_{i^\star}, X_{i^\star}\}] = \mu_{i^\star}$.
	The second inequality holds since $\max\{X_{i^\star}, X_i\} \ge X_{i^\star}$ point-wise, so $\E[\max\{X_{i^\star}, X_i\}] - \mu_{i^\star} \ge 0$ and using 
	\[\max\{a,b\} = \frac{a+b}{2} + \frac{|a-b|}{2}.\]
	Since $X_{i^\star}, X_i \in [0,1]$, we have $|X_{i^\star} - X_i| \in [-1,1]$. Thus, using the i.i.d. assumption, we get:
	\begin{equation}\label{eq:sigma_lb}
		\E[|X_{i^\star} - X_i| ] \ge \E[(X_{i^\star} - X_i)^2]
		=
		\Var(X_{i^\star} - X_i) + \bigl(\E[X_{i^\star} - X_i]\bigr)^2
		\ge
		\sigma_{i^\star}^2 + \sigma_i^2 
		\ge
		\sigma_i^2.
	\end{equation}
	Plugging Equation~\eqref{eq:sigma_lb} into Equation~\eqref{eq:DeltaOPT_lb}, we get:
	\[
	\OPTd - \OPTc
	\ge
	\frac{1}{2n}\sum_{i \neq i^\star } (\sigma_i^2-\Delta_i)_+
	\ge
	\frac{1}{2n}\sum_{i \in \mathcal{S}} (\sigma_i^2-\Delta_i)_+,
	\]
	concluding the proof.
\end{proof}

\subsection{Lower Bound}

\LBDue*

	\begin{proof}
		Let $M:=T-k$. Since $k\le T/100$, we have $M\ge 99T/100$ and, in particular,
		$M\ge 99k$. We prove the result by Yao's principle. Hence, it is enough to fix
		an arbitrary deterministic algorithm $\mathfrak{A}$ and construct a distribution over
		instances on which its expected regret is large.
		
		We split the proof into two cases.
		
		\paragraph{Case 1: $k\le \sqrt{nT}$.}
		In this case, since $M\ge 99T/100$, it is enough to prove a lower bound of order
		$\sqrt{nM}$. Consider the standard stochastic bandit lower-bound construction
		with Bernoulli rewards, where one arm has mean $1-p+\Delta$ and all the other
		arms have mean $1-p$, for a sufficiently small absolute constant $p\in(0,1)$
		and $\Delta = \Theta( \sqrt{{n}/{M}}).$
		
		The usual stochastic bandit lower bound gives regret
		$\Omega(\sqrt{nM})$ on the $M$ non-query rounds. On each query round, the regret
		can be negative, but its magnitude is at most the gap between the dynamic and
		static optimum, which is at most $p$. Therefore, the total negative contribution
		of all query rounds is at most $kp$. Since $k\le \sqrt{nT}$ and $M\ge 99T/100$,
		choosing $p$ as a sufficiently small absolute constant:
		\begin{align*}
			R_{T,k}(A)
			&\ge
			\Omega(\sqrt{nM}) - kp\\
			&=
			\Omega(\sqrt{nM}).
		\end{align*}
		Since in this regime, it holds:
		\[
		\min\left\{\frac{Tn}{k},\sqrt{n(T-k)}\right\}
		=
		\Theta(\sqrt{nM}),
		\]
		the desired lower bound follows.
		
		\paragraph{Case 2: $k\ge \sqrt{nT}$.}
		We show a lower bound of order $Tn/k$. Since $M=\Theta(T)$, it is enough to
		prove a lower bound of order $Mn/k$.
		Fix the tie-breaking rule of the oracle. Let $J$ be the set of the
		$\lfloor n/2\rfloor$ arms with lowest tie-breaking priority. Since the learner
		pulls one arm in each non-query round, $\sum_{i\in J}\mathbb{E}_1[N_i]\le M.$
		Therefore, there exists an arm $j\in J$ such that $\mathbb{E}_1[N_j]\le {3M}/{n}.$
		Here $N_j$ denotes the number of times arm $j$ is pulled in the $M$ non-query
		rounds, and $\mathbb{E}_1$ is expectation under the first instance.
		
		We set the following constant:
		\[
		\Delta := \frac{n}{1000k},
		\qquad
		\epsilon := \frac{M\Delta}{50k},
		\qquad
		p:=2\Delta+\epsilon.
		\]
		Since $k\ge \sqrt{nT}$ and $k\le T/100$, the above parameters satisfy
		$p\le 1/4$ and $1-p+2\Delta\le 1$ up to changing only universal constants.
		
		Consider the base instance $\nu_1$ in which arm $1$ has reward distribution $\operatorname{Be}(1-p+\Delta),$ while every arm $i\neq 1$ has reward distribution
		$\operatorname{Be}(1-p).$
		Thus arm $1$ is the unique optimal arm in $\nu_1$, with gap $\Delta$.
		Now define a second instance $\nu_2$, identical to $\nu_1$ except that arm $j$
		has reward distribution $\operatorname{Be}(1-p+2\Delta).$
		Thus arm $j$ is the unique optimal arm in $\nu_2$, again with gap $\Delta$ over
		arm $1$.
		
		Let $R_T^{(1)}(\mathfrak{A})$ and $R_T^{(2)}(\mathfrak{A})$ be the regret of $\mathfrak{A}$ under $\nu_1$ and
		$\nu_2$, respectively. On query rounds, the learner receives the dynamic
		optimum. Under $\nu_1$, the static optimum is $1-p+\Delta$, so each query
		contributes regret at least $-(p-\Delta)$. Under $\nu_2$, the static optimum is
		$1-p+2\Delta$, so each query contributes regret at least $-(p-2\Delta)$.

		Let $E:=\{N_1\le M/2\}.$ On $\nu_1$, if $E$ occurs, the learner pulls suboptimal arms at least $M/2$
		times in the non-query rounds. Hence:
		\begin{align} \label{eq1:lower}
			R_T^{(1)}(\mathfrak{A})
			\ge
			-k(p-\Delta)
			+
			\frac{M\Delta}{2}\mathbb{P}_1(E).
		\end{align}
		On $\nu_2$, if $E^c$ occurs, the learner pulls arm $1$ at least $M/2$ times,
		and arm $1$ is suboptimal by gap $\Delta$. Thus,
		\begin{align} \label{eq2:lower}
			R_T^{(2)}(\mathfrak{A})
			\ge
			-k(p-2\Delta)
			+
			\frac{M\Delta}{2}\mathbb{P}_2(E^c).
		\end{align}
		Summing Equations~\eqref{eq1:lower}--\eqref{eq2:lower} and applying Bretagnolle--Huber, we get:
		\begin{align}
			R_T^{(1)}(\mathfrak{A})+R_T^{(2)}(\mathfrak{A})
			&\ge
			-k(\Delta+2\epsilon)
			+
			\frac{M\Delta}{2}
			\bigl(\mathbb{P}_1(E)+\mathbb{P}_2(E^c)\bigr) \nonumber\\
			&\ge
			-k(\Delta+2\epsilon)
			+
			\frac{M\Delta}{4}
			\exp\bigl(-D_{\mathrm{KL}}(\mathbb{P}_1\Vert\mathbb{P}_2)\bigr) \label{eq_reg_lower}.
		\end{align}
		It remains to bound the KL divergence between the two feedback histories. By
		the chain rule for KL divergences, it is easy to see that:
		\[
		D_{\mathrm{KL}}(\mathbb{P}_1\Vert\mathbb{P}_2)
		\le
		kD_{\mathrm{KL}}(Q_1\Vert Q_2)
		+
		\mathbb{E}_1[N_j]\,
		D_{\mathrm{KL}}\bigl(\operatorname{Be}(1-p),\operatorname{Be}(1-p+2\Delta)\bigr),
		\]
		where $Q_1,Q_2$ are the distributions of the feedback received in one query
		round under $\nu_1,\nu_2$.
		
		First, for the bandit feedback term, since:
		\[
		D_{\mathrm{KL}}(\operatorname{Be}(1-p),\operatorname{Be}(1-p+2\Delta))
		=
		D_{\mathrm{KL}}(\operatorname{Be}(p),\operatorname{Be}(p-2\Delta))
		=
		D_{\mathrm{KL}}(\operatorname{Be}(p),\operatorname{Be}(\epsilon)),
		\]
		and $p-\epsilon=2\Delta$, the standard Bernoulli KL bound gives:
		\[
		D_{\mathrm{KL}}(\operatorname{Be}(p),\operatorname{Be}(\epsilon))
		\le
		\frac{8\Delta^2}{\epsilon}
		=
		\frac{400k\Delta}{M}
		=
		\frac{2}{5}\frac{n}{M}.
		\]
		By the definition of $\epsilon=M\Delta/50k$. Therefore, it holds
		\begin{equation} \label{eq3_lower}
			\mathbb{E}_1[N_j]\,
			D_{\mathrm{KL}}\bigl(\operatorname{Be}(1-p),\operatorname{Be}(1-p+2\Delta)\bigr)
			\le
			\frac{3M}{n}\cdot \frac{2}{5}\frac{n}{M}
			=
			\frac{6}{5}.
		\end{equation}
		We now bound the query term. Since $j$ belongs to the lower half of the
		tie-breaking order, there are at least $n/2-1$ arms with higher priority than
		$j$. If at least one of these arms obtains reward $1$, then the oracle never
		returns $j$, independently of whether the instance is $\nu_1$ or $\nu_2$.
		Therefore, the query feedback can differ under $\nu_1$ and $\nu_2$ only on the
		event that all these higher-priority arms obtain reward $0$, otherwise the accociated KL is $0$. This event,which will denote as $E$, has
		probability at most $p^{n/2-1}$.
		
		Conditioned on $E$, the only remaining difference between $\nu_1$ and
		$\nu_2$ is the distribution of arm $j$, which changes from
		$\operatorname{Be}(1-p)$ to $\operatorname{Be}(1-p+2\Delta)$. Therefore the
		conditional KL divergence is at most the KL divergence between the two possible
		laws of a single draw from arm $j$:
		\[
		D_{\mathrm{KL}}\bigl(Q_1(\cdot\mid E)\Vert Q_2(\cdot\mid E)\bigr)
		\le
		D_{\mathrm{KL}}\bigl(\operatorname{Be}(1-p),\operatorname{Be}(1-p+2\Delta)\bigr).
		\]
		Since the event occurs with probability at most $p^{n/2-1}$, we get:
		\[
		D_{\mathrm{KL}}(Q_1\Vert Q_2) \le p^{n/2-1}D_{\mathrm{KL}}\bigl(Q_1(\cdot\mid E)\Vert Q_2(\cdot\mid E)\bigr) 
		\le
		p^{n/2-1}\,
		D_{\mathrm{KL}}\bigl(\operatorname{Be}(1-p),\operatorname{Be}(1-p+2\Delta)\bigr).
		\]
		Using the previous KL bound,
		\[
		D_{\mathrm{KL}}(Q_1\Vert Q_2)
		\le
		p^{n/2-1}\frac{2}{5}\frac{n}{M}.
		\]
		Thus, since $M\ge 99k$ and $p\le 1/4$, we get the final bound:
		\begin{equation}\label{eq4_lower}
			kD_{\mathrm{KL}}(Q_1\Vert Q_2)
			\le
			\frac{2n}{5}\frac{k}{M}p^{n/2-1}
			\le
			\frac{2n}{495}p^{n/2-1}
			\le
			\frac{1}{20},
		\end{equation}
		where the last inequality holds for all $n\ge 4$, while the cases $n<4$ are
		absorbed into the constants.
		Combining Equation~\eqref{eq3_lower} and Equation~\eqref{eq4_lower}, we get:
		\[
		D_{\mathrm{KL}}(\mathbb{P}_1\Vert\mathbb{P}_2)
		\le
		\frac{6}{5}+\frac{1}{20}
		=
		\frac{5}{4}.
		\]
		Therefore,
		$    \exp\bigl(-D_{\mathrm{KL}}(\mathbb{P}_1\Vert\mathbb{P}_2)\bigr)
		\ge
		e^{-5/4}.
		$
		Plugging this into Equation~\eqref{eq_reg_lower} gives:
		\begin{align*}
			R_T^{(1)}(\mathfrak{A})+R_T^{(2)}(\mathfrak{A})
			&\ge
			-k(\Delta+2\epsilon)
			+
			\frac{M\Delta}{4e^{5/4}}\\
			& = -k\Delta-\frac{M\Delta}{25}
			+
			\frac{M\Delta}{4e^{5/4}}\\
			& \geq -\frac{M\Delta}{99}-\frac{M\Delta}{25}
			+
			\frac{M\Delta}{4e^{5/4}}\\
			& \geq -\frac{M\Delta}{20}
			+
			\frac{M\Delta}{4e^{5/4}}\\
			& = \Omega(M\Delta),
		\end{align*}
		where we used $M\ge 99k$.
		Since $\Delta=n/(1000k)$ and $M=\Theta(T)$, we obtain:
		\[
		R_T^{(1)}(\mathfrak{A})+R_T^{(2)}(\mathfrak{A})
		=
		\Omega\left(\frac{Tn}{k}\right).
		\]
		Thus, under the uniform prior over $\nu_1,\nu_2$, there exists one of the two
		instances on which
		$    R_{T,k}(\mathfrak{A})
		=
		\Omega(\nicefrac{Tn}{k}).
		$
		Combining the two cases, it holds:
		\[
		R_{T}(\mathfrak{A})
		\ge
		\Omega\left(
		\min\left\{
		\frac{Tn}{k},
		\sqrt{n(T-k)}
		\right\}
		\right),
		\]
		which concludes the proof.
	\end{proof}

\end{document}